\documentclass[10pt]{article}
\usepackage {amsthm}
\usepackage[hmargin=1.4in]{geometry}
\usepackage{natbib}

\newcommand{\emcite}{\citet}
\renewcommand{\cite}{\citep}
\newcommand{\mybibstyle}{plainnat}

\newtheorem {lemma} {Lemma}
\newtheorem {theorem}{Theorem}
\newtheorem {corollary}{Corollary}

\title{Adaptive Bound Optimization\\
 for Online Convex Optimization}

\author{H. Brendan McMahan\\ 
Google, Inc. \\
\texttt{\small mcmahan@google.com}
\and
Matthew Streeter \\
Google, Inc. \\
\texttt{\small mstreeter@google.com}
}

\usepackage{amsfonts}
\usepackage{amsmath}
\usepackage[psamsfonts]{amssymb}
\usepackage{latexsym}
\usepackage{color}
\usepackage{graphics}
\usepackage{enumerate}
\usepackage{amstext}
\usepackage{url}
\usepackage{epsfig}
\usepackage{times}
\usepackage{xspace}
\usepackage{subfigure}
\usepackage{caption}
\usepackage{algorithm}
\usepackage{algorithmic}

\def \argmin {\mathop{\rm arg\,min}}

\def \reals {\mathbb{R}}

\newcommand {\norm}[1]{\ensuremath{\| #1 \|}}

\newcommand {\paren}[1]{\ensuremath{\left(#1\right)}}
\newcommand {\set}[1]{\ensuremath{\left\{#1\right\}}}

\newcommand{\BO}{\mathcal{O}}

\newcommand{\R}{\ensuremath{\mathbb{R}}}
\newcommand{\Regret}{\mathcal{R}}
\newcommand{\grad}{\triangledown}
\newcommand{\Proj}{P}
\newcommand{\abs}[1]{|#1|}

\newcommand{\mleq}{\preceq}
\newcommand{\mgeq}{\succeq}
\newcommand{\mlt}{\prec}
\newcommand{\mgt}{\succ}
\newcommand{\Snpp}{S^n_{++}}

\newcommand{\Qtt}{Q_{1:t}}

\newcommand{\gtt}{g_{1:t}}

\newcommand{\hx}{\hat{x}}
\newcommand{\rtt}{r_{1:t}}
\newcommand{\rt}{r_t}
\newcommand{\uti}{u_{t+1}}
\newcommand{\ut}{u_t}
\newcommand{\xs}{\mathring{x}} %
\newcommand{\xti}{x_{t+1}}
\newcommand{\xt}{x_t}

\newcommand{\us}{\mathring{u}}

\newcommand{\Fs}{\mathcal{F}}
\newcommand{\PFA}{\Proj_{\Fs,A}}

\newcommand{\Fsym}{\Fs_{\text{sym}}}

\newcommand{\uu}{\hat{u}}
\newcommand{\xx}{\hat{x}}

\newcommand{\hy}{\hat{y}}
\newcommand{\hg}{\hat{g}}

\renewcommand{\Regret}{\text{Regret}}

\newcommand{\vs}[1]{\vec{#1}}

\newcommand{\lbar}{\bar{\lambda}}
\newcommand{\abar}{\bar{\alpha}}

\newcommand{\inst}{\mathcal{I}}
\newcommand{\hinst}{\hat{\mathcal{I}}}
\newcommand{\hFs}{\hat{\Fs}}

\newcommand{\Snp}{S^n_{+}}
\renewcommand{\Snpp}{S^n_{++}}

\newcommand{\Qfs}{\mathcal{Q}}

\newcommand{\constQ}{\Qfs_{\text{const}}}
\newcommand{\diagQ}{\Qfs_{\text{diag}}}
\newcommand{\fullQ}{\Qfs_{\text{full}}}

\newcommand{\FTPRL}{FTPRL\xspace}
\newcommand{\FTPRLdiag}{\mbox{FTPRL-Diag}\xspace}
\newcommand{\FTPRLscale}{\mbox{FTPRL-Scale}\xspace}
\newcommand{\E}[1]{\mathbb{E}\left[#1\right]}
\newcommand{\tp}{\top}

\def \oset {\paren}

\newcommand{\shortlong}[2]{#2}   %

\begin{document}

\maketitle

\begin{abstract}
We introduce a new online convex optimization algorithm that
adaptively chooses its regularization function based on the loss
functions observed so far.  This is in contrast to previous algorithms
that use a fixed regularization function such as $L_2$-squared, and
modify it only via a single time-dependent parameter.  Our algorithm's
regret bounds are worst-case optimal, and for certain realistic
classes of loss functions they are much better than existing bounds.
These bounds are problem-dependent, which means they can exploit the
structure of the actual problem instance.  Critically, however, our
algorithm does not need to know this structure in advance.  Rather, we
prove competitive guarantees that show the algorithm provides a bound
within a constant factor of the best possible bound (of a certain
functional form) in hindsight.
\end{abstract}

\section{Introduction}
We consider online convex optimization in the full information
feedback setting.  A closed, bounded convex feasible set $\Fs
\subseteq \reals^n$ is given as input, and on each round $t=1, \dots,
T$, we must pick a point $\xt \in \Fs$.  A convex loss function $f_t$
is then revealed, and we incur loss $f_t(\xt)$.  Our regret at
the end of $T$ rounds is
\begin{equation}
\Regret \equiv \sum_{t=1}^T f_t (\xt) - \min_{x \in \Fs} \sum_{t=1}^T f_t(x) .
\end{equation}

Existing algorithms for online convex optimization are worst-case
optimal in terms of certain fundamental quantities.  In particular,
online gradient descent attains a bound of $\BO(DM\sqrt{T})$ where $D$
is the $L_2$ diameter of the feasible set and $M$ is a bound on
$L_2$-norm of the gradients of the loss functions.  This bound is
tight in the worst case, in that it is possible to construct problems
where this much regret is inevitable.  However, this does not mean
that an algorithm that achieves this bound is optimal in a practical
sense, as on easy problem instances such an algorithm is still allowed
to incur the worst-case regret.  In particular, although this bound is
minimax optimal when the feasible set is a hypersphere \cite
{abernethy08}, we will see that much better algorithms exist when the
feasible set is the hypercube.

To improve over the existing worst-case guarantees,
we introduce additional parameters that capture more of the
problem's structure.
These parameters depend on the loss functions, which are not
known in advance.
To address this, we
first construct functional upper bounds on regret $B_R(\theta_1,
\dots, \theta_T; f_1, \dots, f_T)$ that depend on both
(properties of) the loss functions $f_t$ and algorithm parameters $\theta_t$.
We then give algorithms for choosing the
parameters $\theta_t$ adaptively (based only on $f_1, f_2, \ldots, f_{t-1}$)
and prove that these adaptive schemes provide a
regret bound that is only a constant factor worse than the best
possible regret bound of the form $B_R$.  Formally, if for all
possible function sequences $f_1, \dots f_T$ we have
\[
B_R(\theta_1, \dots, \theta_T; f_1, \dots, f_T) \leq \kappa 
\inf_{\theta_1', \dots, \theta_T' \in \Theta^T} 
B_R(\theta_1', \dots, \theta_T'; f_1, \dots, f_T)
\]
for the adaptively-selected $\theta_t$, we say the adaptive scheme is
$\kappa$-competitive for the bound optimization problem.
In
Section~\ref{sec:motivation}, we provide realistic examples where known bounds
are much worse than the problem-dependent bounds obtained by our
algorithm.

\subsection{Follow the proximally-regularized leader}
We analyze a \emph {follow the regularized leader}
(FTRL) algorithm
that adaptively selects regularization functions of the
form
\[r_t(x) = \frac 1 2 \norm{ (Q_t^{\frac 1 2}(x  - x_t)}^2_2\]
where $Q_t$ is a positive semidefinite matrix.
Our algorithm plays $x_1 = 0$ on round 1 (we assume without loss
of generality that $0 \in \Fs$), and on round $t+1$, selects the point
\begin{equation}\label{eq:ftrlx}
  \xti = \argmin_{x \in \Fs} \oset{ 
     \sum_{\tau=1}^t \big(r_\tau(x) + f_\tau(x)\big)  }.
\end{equation}
In contrast to other FTRL algorithms, such as the dual averaging
method of \emcite{xiao09dualaveraging}, we center the additional
regularization at the current feasible point $x_t$ rather than at the
origin.  Accordingly, we call this algorithm \emph{follow the
proximally-regularized leader} (\FTPRL).  This proximal centering of
additional regularization is similar in spirit to the optimization
solved by online gradient descent (and more generally, online mirror
descent, \cite{cesabianchi06plg}).  However, rather than considering
only the current gradient, our algorithm considers the sum of all
previous gradients, and so solves a global rather than local
optimization on each round.  We discuss related work in more detail in
Section~\ref{sec:related}.

The \FTPRL algorithm allows a clean analysis from first principles,
which we present in Section~\ref{sec:alg_analysis}.
The proof
techniques are rather different from those used for online gradient
descent algorithms, and will likely be of independent interest.

We write $\vs{Q_T}$ as shorthand for $(Q_1, Q_2, \dots, Q_T)$, with
$\vs{g_T}$ defined analogously.  For a convex set $\Fs$, we define
$\Fsym = \set{ x - x' \mid x, x' \in \Fs}$. Using this notation, we
can state our regret bound as
\begin{equation}\label{eq:brfunc}
\Regret 
 \leq B_R(\vs{Q_T}, \vs{g_T}) 
 \equiv \frac{1}{2} \sum_{t=1}^T \max_{\hy \in \Fsym} \oset { \hy^\tp Q_t\hy }
          + \sum_{t=1}^T g_t^\tp \Qtt^{-1}g_t
\end{equation}
where $g_t$ is a subgradient of $f_t$ at $x_t$ and $\Qtt =
\sum_{\tau=1}^t Q_\tau$.  We prove competitive ratios with respect to
this $B_R$ for several adaptive schemes for selecting the $Q_t$
matrices.  In particular, when the \FTPRLdiag scheme is run on a
hyperrectangle (a set of the form $\set{x \mid x_i
\in [a_i, b_i]} \subseteq \R^n)$, we achieve
\[ \Regret 
\leq  \sqrt{2} \inf_{\vs{Q} \in \diagQ^T} B_R(\vs{Q_T}, \vs{g_T})
\]
where $\diagQ = \set{ \text{diag}(\lambda_1, \dots, \lambda_n) \mid
\lambda_i \geq 0}$.  When the \FTPRLscale scheme is run on a
feasible set of the form $\Fs = \set{x \mid \norm{Ax}_2 \leq 1}$ for $A \in \Snpp$,
it is competitive with arbitrary positive semidefinite matrices:
\[ \Regret \leq \sqrt{2} 
\inf_{\vs{Q} \in (\Snp)^T} B_R(\vs{Q_T}, \vs{g_T})\ .
\]

Our analysis of \FTPRL reveals a fundamental connection between the
shape of the feasible set and the importance of choosing the
regularization matrices adaptively.  When the feasible set is a
hyperrectangle, \FTPRLdiag has stronger bounds than known algorithms,
except for degenerate cases where the bounds are identical.  In
contrast, when the feasible set is a hypersphere, $\set{x
\mid \norm{x}_2 \leq 1}$, the bound $B_R$ is always optimized by
choosing $Q_t = \lambda_t I$ for suitable $\lambda_t \in \R$.  The
\FTPRLscale scheme extends this result to hyperellipsoids by applying
a suitable transformation.  These results are presented in detail in
Section~\ref{sec:ratios}.

\subsection{The practical importance of adaptive regularization}
\label{sec:motivation}

In the past few years, online algorithms have emerged as
state-of-the-art techniques for solving large-scale machine learning
problems~\cite{bottou08,zhang04}.  Two canonical
examples of such large-scale learning problems are text classification
on large datasets and predicting click-through rates for ads on a
search engine.  For such problems, extremely large feature sets may be
considered, but many features only occur rarely, while few occur very
often.  Our diagonal-adaptation algorithm offers improved bounds for
problems such as these.

As an example, suppose $\Fs = [-\frac{1}{2}, \frac{1}{2}]^n$ (so
$D=\sqrt{n}$).  On each round $t$, the $i$th component of $\grad f_t(x_t)$
(henceforth $g_{t, i}$) is 1 with
probability $i^{-\alpha}$, and is 0 otherwise, for some $\alpha \in
[1, 2)$.  Such heavy-tailed distributions are common in text
classification applications, where there is a feature for each word.
In this case, gradient descent with a global learning
rate\footnote{The $\BO(DM \sqrt{T})$ bound (mentioned in the
introduction) based on a $1/\sqrt{t}$ learning rate gives $\BO(n
\sqrt{T})$ here; to get $\BO(\sqrt{nT})$ a global rate based on
$\norm{g_t^2}$ is needed, e.g., Corollary~\ref{cor:global}.}  obtains
an expected regret bound of $O(\sqrt {n T})$.  In contrast, the
algorithms presented in this paper will obtain expected regret on the order of
\[
\E { \sum_{i=1}^n \sqrt { \sum_{t=1}^T g_{t, i}^2 } }
\le \sum_{i=1}^n \sqrt { \sum_{t=1}^T \E {  g_{t, i}^2  } }
= \sum_{i=1}^n \sqrt { T i^{-\alpha}  }
= O(\sqrt {T} \cdot n^{1 - \frac \alpha 2} )
\]
using Jensen's inequality.  This bound is never worse than the
$O(\sqrt {n T})$ bound achieved by ordinary gradient descent, and can
be substantially better.  For example, in problems where a constant
fraction of examples contain a new feature, $n$ is $\Omega(T)$ and the
bound for ordinary gradient descent is vacuous.  In contrast, the
bound for our algorithm is $O(T^{\frac {3-\alpha} {2}})$, which is
sublinear for $\alpha > 1$.

This performance difference is not merely a weakness in the regret
bounds for ordinary gradient descent, but is a difference in actual
regret.  In concurrent work \cite{streeter10percoord}, we showed that
for some problem families, a per-coordinate learning rate for online
gradient descent provides asymptotically less regret than even the
best non-increasing global learning rate (chosen in hindsight, given
the observed loss functions).  This construction can be adapted to \FTPRL as:
\begin{theorem} \label {thm:bad_class}
There exists a family of online convex optimization problems,
parametrized by the number of rounds $T$, where online subgradient
descent with a non-increasing learning rate sequence (and \FTPRL with
non-decreasing coordinate-constant regularization) incurs regret at
least $\Omega(T^\frac{2}{3})$, whereas \FTPRL with appropriate
diagonal regularization matrices $Q_t$ has regret $O(\sqrt{T})$.
\end{theorem}
In fact, any online learning algorithm whose regret is
$O(M D \sqrt T)$ (where $D$ is the $L_2$ diameter of the
feasible region, and $M$ is a bound on the $L_2$ norm of the
gradients) will suffer regret $\Omega(T^\frac{2}{3})$ on this family
of problems.
Note that this does not contradict the $O(M D \sqrt T)$ upper bound on
the regret, because in this family of problems $D = T^{\frac 1 6}$
(and $M = 1$).

\subsection{Adaptive algorithms and competitive ratios}

In Section~\ref{sec:ratios}, we introduce specific schemes for
selecting the regularization matrices $Q_t$ for \FTPRL, and show that
for certain feasible sets, these algorithms provide bounds within a
constant factor of those for the best post-hoc choice of matrices,
namely
\begin{equation} \label{eq:posthocbound}
  \inf_{\vs{Q_T} \in \Qfs^T} B_R(\vs{Q_T}, \vs{g_T})
\end{equation}
where $\Qfs \subseteq \Snp$ is a set of allowed matrices; $\Snp$ is
the set of symmetric positive semidefinite $n \times n$ matrices, with
$\Snpp$ the corresponding set of symmetric positive definite matrices.
We consider three different choices for $\Qfs$: the set of
coordinate-constant matrices $\constQ = \set{\alpha I \mid \alpha \geq
0}$; the set of non-negative diagonal matrices,
\[\diagQ = \set{\text{diag}(\lambda_1, \dots, \lambda_n) \mid \lambda_i \geq 0};\]
and, the full set of positive-semidefinite matrices, $\fullQ =\Snp$.

We first consider the case where the feasible region is an $L_p$ unit
ball, namely $\Fs = \set{x \mid \norm{x}_p \leq 1}$.  For $p \in [1, 2]$,
we show that a simple
algorithm (an analogue of standard online gradient descent) that
selects matrices from $\constQ$ is $\sqrt{2}$-competitive with the
best post-hoc choice of matrices from the full set of positive
semidefinite matrices $\fullQ = \Snp$.  
This algorithm is presented in
Corollary~\ref{cor:global}, and the competitive ratio is proved in
Theorem~\ref{thm:sphere}.

In contrast to the result for $p \in [1, 2]$, we show that for $L_p$
balls with $p > 2$ a coordinate-independent choice of matrices $(Q_t
\in \constQ)$ does not in general obtain the post-hoc optimal bound
(see Section~\ref{sec:lp_posthoc}), and hence per-coordinate
adaptation can help.  The benefit of per-coordinate adaptation is most
pronounced for the $L_\infty$-ball, where the coordinates are
essentially independent.  In light of this, we develop an efficient
algorithm (\FTPRLdiag, Algorithm~\ref{alg:diag}) for adaptively
selecting $Q_t$ from $\diagQ$, which uses scaling based on the width
of $\Fs$ in the coordinate directions (Corollary~\ref{cor:percoord}).
In this corollary, we also show that this algorithm
$\sqrt{2}$-competitive with the best post-hoc choice of matrices from
$\diagQ$ when the feasible set is a hyperrectangle.

While per-coordinate adaptation does not help for the
unit $L_2$-ball, it can help when the feasible set is a hyperellipsoid.
In particular, in the case where $\Fs = \set{x
\mid \norm{A x}_2 \leq 1}$ for $A \in \Snpp$,
we show that an appropriate transformation of the problem can produce
significantly better regret bounds.
More generally, we show (see Theorem~\ref{thm:blackbox}) that if one
has a $\kappa$-competitive adaptive \FTPRL scheme for the feasible set
$\set{x \mid \norm{x} \leq 1}$ for an arbitrary norm, it can be extended to
provide a $\kappa$-competitive algorithm for feasible sets
of the form $\set{x \mid \norm{A x} \leq 1}$.  Using this result, we
can show
\FTPRLscale is $\sqrt{2}$-competitive with the
best post-hoc choice of matrices from $\Snp$ when $\Fs = \set{x \mid
\norm{A x}_2 \leq 1}$ and $A \in \Snpp$; it is $\sqrt{2}$-competitive
with $\diagQ$ when $\Fs = \set{x \mid \norm{A x}_p \leq 1}$ for $p \in
[1,2)$.

Of course, in many practical applications the feasible set may not be
so nicely characterized.  We emphasize that our algorithms and
analysis are applicable to arbitrary feasible sets, but the quality of
the bounds and competitive ratios will depend on how tightly the
feasible set can be approximated by a suitably chosen transformed norm
ball.  In Theorem~\ref{thm:arbcomp}, we show in particular that when
\FTPRLdiag is applied to an arbitrary feasible set, it provides a
competitive guarantee related to the ratio of the widths of the
smallest hyperrectangle that contains $\Fs$ to the largest
hyperrectangle contained in $\Fs$.

\subsection{Notation and technical background}
We use the notation $\gtt$ as a shorthand for $\sum_{\tau=1}^t
g_\tau$. Similarly we write $\Qtt$ for a sum of matrices $Q_t$, and
$f_{1:t}$ to denote the function $f_{1:t}(x) = \sum_{\tau = 1}^t
f_\tau(x)$.  We write $x^\tp y$ or $x \cdot y$ for the inner product
between $x,y \in \R^n$.  The $i$th entry in a vector $x$ is denoted
$x_i \in \R$; when we have a sequence of vectors $x_t \in \R^n$ indexed by time, the
$i$th entry is $x_{t,i} \in \R$.  We use $\partial f(x)$ to denote the set of
subgradients of $f$ evaluated at $x$.

Recall $A \in \Snpp$ means $\forall x\neq 0,\ x^\tp Ax > 0$.  We use
the generalized inequality $A \mgt 0$ when $A
\in \Snpp$, and similarly $A \mlt B$ when $B-A \mgt 0$, implying
$x^\tp Ax < x^\tp Bx$.  We define $A \mleq B$ analogously for
symmetric positive semidefinite matrices $\Snp$.  For $B \in
\Snp$, we write $B^{1/2}$ for the square root of $B$, the unique $X \in
\Snp$ such that  $XX = B$ (see, for example,~\emcite[A.5.2]{boyd}).
We also make use of the fact that any $A \in \Snp$ can be factored as
$A = PDP^\tp $ where $P^\tp P = I$ and $D=\text{diag}(\lambda_1,
\dots, \lambda_n)$ where $\lambda_i$ are the eigenvalues of $A$.  

Following the arguments of \emcite{zinkevich03}, for the remainder we
restrict our attention to linear functions.  Briefly, the convexity of
$f_t$ implies $f_t(x) \ge g_t^\tp (x - x_t) + f_t(x_t)$, where $g_t \in
\partial f(x_t)$.  Because this inequality is tight for $x = x_t$, it follows
that regret measured against the affine functions on the right hand
side is an upper bound on true regret.  Furthermore, regret is
unchanged if we replace this affine function with the linear function
$g_t^\tp x$.  Thus, so long as our algorithm only makes use of the
subgradients $g_t$, we may assume without loss of generality that the
loss functions are linear.

Taking into account this reduction and the functional form of the
$r_t$, the update of \FTPRL is
\begin{equation}\label{eq:linearquadopt}
\xti = \argmin_{x \in \Fs} \oset {\frac{1}{2} 
\sum_{\tau=1}^t (x - x_\tau)^\tp Q_\tau (x - x_\tau) + \gtt \cdot x}.
\end{equation}

\section {Analysis of \FTPRL}\label{sec:alg_analysis}

In this section, we prove the following bound on the regret of \FTPRL
for an arbitrary sequence of regularization matrices $Q_t$.  In this
section $\norm{\cdot}$ always means the $L_2$ norm, $\norm{\cdot}_2$.

\begin{theorem}\label{thm:adaptive_norms}
Let $\Fs \subseteq \R^n$ be a closed, bounded convex set with $0 \in
\Fs$.  Let $Q_1 \in \Snpp$, and $Q_2, \dots, Q_T \in \Snp$.  Define $r_t(x) = \frac 1 2
\norm {Q^{\frac 1 2}_t (x -
\xt)}^2_2$, and $A_t = (Q_{1:t})^{\frac 1 2}$.  Let $f_t$ be a sequence
of loss functions, with $g_t \in \partial f_t(x_t)$ a sub-gradient of
$f_t$ at $x_t$.  Then, the \FTPRL algorithm that that faces loss
functions $f$, plays $x_1 = 0$, and uses the update of
Equation~\eqref{eq:linearquadopt} thereafter, has a regret bound
\[ \Regret \leq r_{1:T}(\xs) + \sum_{t=1}^T \norm{A_t^{-1} g_t}^2\]
where
$\xs = \argmin_{x \in \Fs} f_{1:T}(x)$ is the post-hoc optimal feasible point.
\end{theorem}

To prove Theorem \ref {thm:adaptive_norms} we will make use of the
following bound on the regret of FTRL, which holds for arbitrary
(possibly non-convex) loss functions.  
\shortlong{
This lemma can be proved
along the lines of \cite{kv03}; for a complete proof see
\cite[Appendix A]{mcmahan10long}.}{
This lemma can be proved
along the lines of \cite{kv03}; for completeness, a proof 
is included in Appendix~\ref{sec:ftrl_appendix}.
}
\begin {lemma} \label {lem:ftrl}
Let $r_1, r_2, \ldots, r_T$ be a sequence of non-negative functions.
The regret of \FTPRL (which plays $\xt$ as defined by
Equation~\eqref{eq:ftrlx}) is bounded by
\[
r_{1:T}(\xs) + \sum_{t=1}^T \paren{f_t (\xt) - f_t(\xti)}
\]
where $\xs$ is the post-hoc optimal feasible point.
\end {lemma}

Once Lemma \ref {lem:ftrl} is established, to prove Theorem \ref
{thm:adaptive_norms} it suffices to show that for all $t$,
\begin {equation} \label {eq:toprove}
f_t(\xt) - f_t(\xti) \le \norm{A_t^{-1} g_t}^2.
\end {equation}

To show this, we first establish an alternative characterization of our
algorithm as solving an unconstrained optimization followed by a
suitable projection onto the feasible set.  Define the projection
operator,
\[
\Proj_{\Fs,A}(u) = \argmin_{x \in \Fs} \norm{A(x - u)}
\]
We will show that the following is an equivalent formula for $\xt$:
\begin {align} 
\uti &= \argmin_{u \in \reals^n} \oset { r_{1:t}(u) + \gtt \cdot u } \notag \\
\xti &= \Proj_{\Fs,A_t}\paren {\uti}. \label {eq:ftrlislazyprojection}
\end {align}

This characterization will be useful, because the unconstrained
solutions depend only on the linear functions $g_t$, and the quadratic
regularization, and hence are easy to manipulate in closed form.

To show this equivalence, first note that because $Q_t \in \Snp$ is symmetric,
\[ r_t(u) = \frac{1}{2} (u - x_t)^\tp Q_t(u - x_t) 
          = \frac{1}{2} u^\tp Q_t u - x_t^\tp Q_tu_t + \frac{1}{2} x_t^\tp Q_tx_t.\]

Defining
constants $q_t = Q_t x_t$ and $k_t = \frac{1}{2} x_t^\tp Q_tx_t $, we can
write
\begin{equation}\label{eq:rtt}
 \rtt(u) = \frac{1}{2} u^\tp \Qtt u - q_{1:t} u + k_{1:t}.
\end{equation}

The equivalence is then a corollary of the following lemma,
choosing $Q = \Qtt$ and $h = \gtt - q_{1:t}$ (note that the
constant term $k_{1:t}$ does not influence the argmin).
\begin{lemma}\label{lem:lazy_proj} 
Let $Q \in \Snpp$ and $h \in \R^n$, and consider the function
\[ f(x) = h^\tp x + \frac{1}{2}x^\tp Qx.\]
Let $\us = \argmin_{u \in \R^n} f(u)$.  Then, letting $A = Q^{\frac{1}{2}}$, we have
$\PFA(\us) = \argmin_{x \in \Fs} f(x).$
\end{lemma}

\begin{proof}
Note that $\grad_u f(u) = h + Qu$, implying that $\us = -Q^{-1}h$.
Consider the function 
\[f'(x) = \frac{1}{2}\norm{Q^{\frac{1}{2}}(x - \us)}^2 
        = \frac{1}{2}(x - \us)^\tp Q(x - \us).\]
We have
\begin{align*}
f'(x) &= %
 \frac{1}{2}\Big(x^\tp Qx - 2 x^\tp Q \us + \us^\tp Q\us\Big) & \text{(because $Q$ is symmetric)} \\
 &= \frac{1}{2}\Big(x^\tp Qx + 2 x^\tp Q (Q)^{-1}h + \us^\tp Q \us \Big) \\
 &= \frac{1}{2}\Big(x^\tp Qx + 2 x^\tp h + \us^\tp Q \us \Big) \\
 &= f(x) + \frac{1}{2}\us^\tp Q \us \mbox { .} 
\end{align*}
Because $\frac{1}{2} \us^\tp Q \us$ is constant with respect to $x$, it
follows that
\[
\argmin_{x \in \Fs} f(x) = \argmin_{x \in \Fs} f'(x) =\PFA(\us),
\]
where the last equality follows from the definition of the projection operator.
\end{proof}

We now derive a closed-form solution to the unconstrained problem.  It
is easy to show $\grad \rt(u) = Q_tu - Q_tx_t$, and so
\[\grad \rtt(u) = \Qtt u - \sum_{\tau=1}^tQ_\tau x_\tau.\]
Because $\uti$ is the optimum of the (strongly convex) unconstrained problem,
and $\rtt$
is differentiable,
we must have $\grad
\rtt(\uti) + \gtt = 0$.  Hence, we conclude
$\Qtt \uti - \sum_{\tau=1}^tQ_\tau x_\tau + \gtt = 0,$
or
\begin {equation} \label {eq:unconstrained_opt}
\uti = \Qtt^{-1}\paren{\sum_{\tau=1}^tQ_\tau x_\tau - \gtt}.
\end {equation}

This closed-form solution will let us bound the difference between
$\ut$ and $\uti$ in terms of $g_t$.  The next Lemma relates this
distance to the difference between $\xt$ and $\xti$, which determines
our per round regret (Equation~\eqref{eq:toprove}).  In particular, we
show that the projection operator only makes $\ut$ and $\uti$ closer
together, in terms of distance as measured by the norm $\norm{A_t
\cdot}$.  We defer the proof to the end of the section.

\begin{lemma}\label{lem:a_proj}
Let $Q \in \Snpp$ with $A = Q^\frac{1}{2}$.  Let $\Fs$ be a convex
set, and let $u_1, u_2 \in
\R^n$, with $x_1 = \Proj_{\Fs,A}(u_1)$ and $x_2 = \Proj_{\Fs,A}(u_2)$.  Then,
\[\norm{A (x_2 - x_1)} \leq \norm{A(u_1 - u_2)}.\]
\end{lemma}

We now prove the following lemma, which will immediately yield the
desired bound on $f_t (\xt) - f_t(x_{t+1})$.

\newcommand{\PFAt}{\Proj_{\Fs,A}}

\begin{lemma}\label{lem:not_btarl_term}
Let $Q \in \Snpp$ with $A = Q^\frac{1}{2}$.  Let $v, g \in \R^n$, and
let $u_1 = -Q^{-1} v$ and $u_2 = -Q^{-1} (v + g)$. Then, letting $x_1
= \PFAt(u_1)$ and $x_2 = \PFAt(u_2)$,
\[ g^\tp (x_1 - x_2) \leq \norm{A^{-1} g}^2.\]
\end{lemma}
\begin{proof}
The fact that $Q = A^\tp A \mgt 0$ implies that
$\norm{A \cdot}$ and $\norm{A^{-1}\cdot}$ are dual norms (see
for example~\cite[Sec. 9.4.1, pg. 476]{boyd}).  Using this fact,
\begin{align*}
 g^\tp (x_1 - x_2)
  &\leq \norm{A^{-1} g} \cdot \norm{A (x_1 - x_2)}\\
  &\leq \norm{A^{-1} g} \cdot \norm{A (u_1 - u_2)}
         && \text{(Lemma~\ref{lem:a_proj})}\\
  &= \norm{A^{-1} g} \cdot \norm{A (Q^{-1} g)} \\
  &= \norm{A^{-1} g} \cdot \norm{A (A^{-1}A^{-1}) g)} 
      && \text{(Because $Q^{-1} = (A A)^{-1}$)}\\
  &= \norm{A^{-1} g} \cdot \norm{A^{-1} g}. 
\end{align*}
\end{proof}

\noindent
\textbf{Proof of Theorem~\ref{thm:adaptive_norms}:}
First note that because $r_t(\xt) = 0$ and  $r_t$ is non-negative, 
$\xt = \argmin_{x \in \Fs} r_t(x)$.  For any functions $f$ and
$g$, if $x^* = \argmin_{x \in \Fs} f(x)$ and $x^* = \argmin_{x
\in \Fs} g(x)$, then 
\[x^* = \argmin_{x \in \Fs} \oset {f(x) +g(x)}.\]
Thus we have
\begin {align*}
\xt
& = \argmin_{x \in \Fs} \oset { g_{1:t-1} x + r_{1:t-1}(x) } \\
& = \argmin_{x \in \Fs} \oset { g_{1:t-1} x + r_{1:t}(x) } 
  && \text{(Because $x_t = \argmin_{x \in \Fs} r_t(x)$.)}\\
& = \argmin_{x \in \Fs} \oset { h x + \frac 1 2 x^\tp \Qtt x } \\
\end {align*}
where the last line follows from Equation~\eqref{eq:rtt},
letting $h = g_{1:t-1} - q_{1:t} = g_{1:t-1} -
\sum_{\tau=1}^{t} Q_\tau x_\tau$, and dropping the constant $k_{1:t}$.
For $\xti$, we have directly from the definitions
\begin{align*}
\xti =  \argmin_{x \in \Fs} \oset { g_{1:t} x + r_{1:t}(x) } 
     = \argmin_{x \in \Fs} \oset { (h + g_t) x + \frac 1 2  x^\tp \Qtt x }.
\end{align*}
Thus, Lemma~\ref{lem:lazy_proj} implies $\xt = \Proj_{\Fs, A_t}(
-(\Qtt)^{-1} h)$ and similarly $\xti = \Proj_{\Fs, A_t}( -(\Qtt)^{-1} (h
+ g_t))$.  Thus, by Lemma~\ref{lem:not_btarl_term}, $g_t (\xt - \xti)
\le
\norm{A_t^{-1} g_t}^2$.
The theorem then follows from Lemma~\ref{lem:ftrl}.
\qed

\ \\
\noindent
\textbf{Proof of Lemma~\ref{lem:a_proj}:}
Define
\[B(x, u) = \frac{1}{2}\norm{A(x-u)}^2 = \frac{1}{2}(x-u)^\tp Q(x-u),\]
 so we can write equivalently
\[x_1 = \argmin_{x \in \Fs} B(x, u_1).\]
Then, note that $\grad_x B(x, u_1) = Qx - Qu_1$, and so we must have
$(Qx_1 - Qu_1)^\tp (x_2 - x_1) \geq 0$; otherwise for $\delta$
sufficiently small the point $x_1  + \delta (x_2 - x_1)$
would belong to $\Fs$ (by convexity) and would be closer to $u_1$ than
$x_1$ is.
Similarly, we must have $(Qx_2 - Qu_2)^\tp (x_1 - x_2) \geq 0$.
Combining these, we have the following equivalent inequalities:
\begin{gather*}
(Qx_1 - Qu_1)^\tp (x_2 - x_1) - (Qx_2 - Qu_2)^\tp (x_2 - x_1)\geq 0\\
(x_1 - u_1)^\tp Q(x_2 - x_1) - (x_2 - u_2)^\tp Q(x_2 - x_1)\geq 0\\
-(x_2 - x_1)^\tp Q(x_2 - x_1) + (u_2 - u_1)^\tp Q(x_2 - x_1)\geq 0\\
(u_2 - u_1)^\tp Q(x_2 - x_1) \geq (x_2 - x_1)Q(x_2 - x_1).
\end{gather*}
Letting $\uu = u_2 - u_1$, and $\xx = x_2 - x_1$, we have $\xx^\tp Q\xx
\leq \uu^\tp Q\xx$.  Since $Q$ is positive semidefinite, we
have $(\uu - \xx)^\tp Q(\uu - \xx) \geq 0$, or equivalently $\uu^\tp Q\uu +
\xx^\tp Q\xx - 2 \xx^\tp Q\uu \geq 0$ (using the fact $Q$ is also symmetric).  Thus,
\begin{align*}
\uu^\tp Q\uu  
    \ \geq \ - \xx^\tp Q\xx + 2 \xx^\tp Q\uu 
    \ \geq \ - \xx^\tp Q\xx + 2 \xx^\tp Q\xx 
    \  =  \ \xx^\tp Q\xx,
\end{align*}
and so
\[\norm{A(u_2 - u_1)}^2 = \uu^\tp Q\uu \geq  \xx^\tp Q\xx = \norm{A(x_2 - x_1)}^2.\]
\qed

\section{Specific Adaptive Algorithms and Competitive Ratios}
\label{sec:ratios}

Before proceeding to the specific results, we establish several
results that will be useful in the subsequent arguments.  In order to
prove that adaptive schemes for selecting $Q_t$ have good competitive
ratios for the bound optimization problem, we will need to compare the
bounds obtained by the adaptive scheme to the optimal post-hoc bound
of Equation~\eqref{eq:posthocbound}.  Suppose the sequence $Q_1,
\dots, Q_T$ is optimal for Equation~\eqref{eq:posthocbound}, and
consider the alternative sequence $Q_1' = Q_{1:T}$ and $Q_t' = 0$ for
$t > 1$.  Using the fact that $Q_{1:t}
\mgeq Q_{1:t-1}$ implies $Q_{1:t}^{-1} \mleq Q_{1:t-1}^{-1}$, it is easy to
show the alternative sequence also achieves the minimum.
It follows that a sequence with $Q_1 = Q$ on the first round, and $Q_t
= 0$ thereafter is always optimal.  Hence, to solve for the post-hoc
bound we can solve an optimization of the form
\begin{equation}\label{eq:fixedQ}
 \inf_{Q \in \Qfs} \quad \left(\max_{\hy \in \Fsym} \oset { \frac{1}{2} \hy^\tp Q\hy }
             + \sum_{t=1}^T g_t^\tp Q^{-1}g_t\right).
\end{equation}
The diameter of $\Fs$ is $D \equiv \max_{y, y' \in \Fs} \norm{y -
y'}_2$, and so for $\hy \in \Fsym$, $\norm{\hy}_2 \leq D$.  
When $\Fs$ is symmetric ($x \in \Fs$ implies $-x \in \Fs$), we have $y
\in \Fs$ if and only if $2 y \in \Fsym$, so \eqref{eq:fixedQ} is
equivalent to:
\begin{equation}\label{eq:fixedQF}
 \inf_{Q \in \Qfs} \quad  \left(\max_{y \in \Fs} \oset { 2 y^\tp Qy }
             + \sum_{t=1}^T g_t^\tp Q^{-1}g_t\right).
\end{equation}
For simplicity of exposition, we assume $g_{1, i} > 0$ for all $i$,
which ensures that only positive definite
matrices can be optimal.\footnote{In the case where $\Fs$ has 0 width
in some direction, the infimum will not be attained by a finite $Q$,
but by a sequence that assigns $0$ penalty (on the right-hand side) to
the components of the gradient in the direction of 0 width, requiring
some entries in $Q$ to go to $\infty$.}
This assumption also ensures $Q_1 \in \Snpp$ for the adaptive schemes
discussed below, as required by Theorem~\ref{thm:adaptive_norms}.
This is without loss of generality, as we can always hallucinate
an initial loss function with arbitrarily small components, and this changes
regret by an arbitrarily small amount.
\shortlong{We will also use the following Lemma~\cite{auer00adaptive}:}{
We will also use the following Lemma from~\emcite{auer00adaptive}.  For
completeness, a proof is included in Appendix~\ref{ap:lem_sum_proof}.
}
\begin {lemma} \label {lem:sum}
For any non-negative real numbers $x_1, x_2, \ldots, x_n$,
\[
\sum_{i=1}^n \frac { x_i } { \sqrt { \sum_{j=1}^i x_j } }
  \le 2 \sqrt { \sum_{i=1}^n x_i } \mbox { .}
\]
\end {lemma}

\subsection {Adaptive coordinate-constant regularization}
\label{sec:global_reg}
We derive bounds where $Q_t$ is chosen from the set $\constQ$, and
show that this algorithm comes within a factor of $\sqrt{2}$ of using
the best constant regularization strength $\lambda I$.  This algorithm
achieves a bound of $\BO(DM\sqrt{T})$ where $D$ is the diameter of
the feasible region and $M$ is a bound on $\norm{g_t}_2$, matching the
best possible bounds in terms of these parameters \cite{abernethy08}. 
We will prove a much stronger competitive guarantee for this algorithm
in Theorem~\ref{thm:sphere}.

\begin{corollary}\label{cor:global}
Suppose $\Fs$ has $L_2$ diameter $D$.  Then, if we run \FTPRL with
diagonal matrices such that
\[ (Q_{1:t})_{ii} = \abar_t = \frac{2\sqrt{G_t}}{D}\]
 where $G_t = \sum_{\tau=1}^t \sum_{i=1}^n g_{\tau,i}^2$, then
\[ \Regret \leq  2 D \sqrt{G_T}.\]
If $\norm{g_t}_2 \leq M$, then $G_T \leq M^2 T$, and this translates
to a bound of $\BO(DM\sqrt{T})$.  When $\Fs = \set{x \mid \norm{x}_2
\leq D/2}$, this bound is $\sqrt{2}$-competitive for the bound
optimization problem over $\constQ$.
\end{corollary}

\begin{proof}
Let the diagonal entries of $Q_t$ all be $\alpha_t = \abar_{t} -
\abar_{t-1}$ (with $\abar_0 = 0$), so $\alpha_{1:t} = \abar_{t}$.
Note $\alpha_t \geq 0$, and so this choice is feasible.
We consider the left and right-hand terms of
Equation~\eqref{eq:brfunc} separately.  For the left-hand term,
letting $\hy_t$ be an arbitrary sequence of points from $\Fsym$, and
noting $\hy_t^\tp \hy_t \leq
\norm{\hy_t}_2 \cdot \norm{\hy_t}_2 \leq D^2$,
\[ 
   \frac{1}{2} \sum_{t=1}^T \hy_t^\tp Q_t\hy_t 
  = \frac{1}{2}\sum_{t=1}^T  \hy_t^\tp \hy_t \alpha_t
  \leq \frac{1}{2}D^2 \sum_{t=1}^T\alpha_t
  =  \frac{1}{2}D^2 \abar_T
  = D\sqrt{G_T}.
\] 
For the right-hand term, we have 
\[
 \sum_{t=1}^T g_t^\tp Q_{1:t}^{-1}g_t
  = \sum_{t=1}^T \sum_{i=1}^n \frac{g_{t,i}^2}{\alpha_{1:t}}
  = \sum_{t=1}^T \frac{D}{2} \frac{\sum_{i=1}^n g_{t,i}^2}{\sqrt{G_t}}
 \leq D \sqrt{G_T},
\]
where the last inequality follows from Lemma~\ref{lem:sum}.  

In order to make a competitive guarantee, we must prove a lower bound
on the post-hoc optimal bound function $B_R$,
Equation~\eqref{eq:fixedQ}.  This is in contrast to the upper bound
that we must show for the regret of the algorithm.  When $\Fs =
\set{x \mid \norm{x}_2 \leq D/2}$, Equation~\eqref{eq:fixedQ} simplifies to exactly
\begin{equation}\label{eq:fixed_posthoc}
\min_{\alpha \geq 0} \oset { \frac{1}{2} \alpha D^2
             + \frac{1}{\alpha} G_T } = D\sqrt{2G_T}
\end{equation}
and so we conclude the adaptive algorithm is $\sqrt{2}$-competitive
for the bound optimization problem.
\end{proof}
\begin{figure*}[ttt!]
 \begin{minipage}[t]{2.8in} \begin{algorithm}[H]
\caption{\FTPRLdiag} 
\label{alg:diag} 
\begin{algorithmic}
   \STATE {\bfseries Input:} feasible set $\Fs \subseteq \times_{i=1}^n [a_i, b_i]$
   \STATE Initialize $x_1 = 0 \in \Fs$
   \STATE $(\forall i),\ G_i = 0, q_i = 0, \lambda_{0,i} = 0, D_i = b_i - a_i$ 
   \FOR{$t=1$ {\bfseries to} $T$}
   \STATE Play the point $x_t$, incur loss $f_t(x_t)$
   \STATE Let $g_t \in \partial f_t(x_t)$
   \FOR{$i=1$ \textbf{to} $n$}
     \STATE $G_i = G_i + g_{t,i}^2$
     \STATE $\lambda_{t,i} = \frac{2}{D_i}\sqrt{G_i} - \lambda_{1:t-1,i}$
     \STATE $q_i = q_i + x_{t,i}\lambda_{t,i}$
     \STATE $u_{t+1, i} = (g_{1:t,i} - q_i)/\lambda_{1:t,i}$
   \ENDFOR
   \STATE $A_t = \text{diag}(\sqrt{\lambda_{1:t,1}}, \dots, 
           \sqrt{\lambda_{1:t,n}})$
   \STATE $x_{t+1} = \text{Project}_{\Fs, A_t}(u_{t+1})$
   \ENDFOR
\end{algorithmic}
\end{algorithm}
\end{minipage}
\hfill
\begin{minipage}[t]{2.8in}
\begin{algorithm}[H]
\caption{\FTPRLscale} 
\label{alg:scale} 
\begin{algorithmic}
   \STATE {\bfseries Input:} 
          feasible set $\Fs \subseteq \set{x \mid \norm{A x} \leq 1}$, \\
   \STATE \ \ with $A \in \Snpp$
   \STATE Let $\hFs = \set{x \mid \norm{x} \leq 1}$
   \STATE Initialize $x_1 = 0$, $(\forall i)\ D_i = b_i - a_i$
   \FOR{$t=1$ {\bfseries to} $T$}
   \STATE Play the point $x_t$, incur loss $f_t(x_t)$
   \STATE Let $g_t \in \partial f_t(x_t)$
   \STATE $\hg_t = (A^{-1})^\tp g_t$
   \STATE $\abar = \sqrt{\sum_{\tau=1}^t \sum_{i=1}^n \hg_{\tau,i}^2}$
   \STATE $\alpha_t = \abar - \alpha_{1:t-1}$
   \STATE $q_t = \alpha_t  x_t$
   \STATE $\hat{u}_{t+1} = (1/\abar)(q_{1:t} - g_{1:t})$
   \STATE $A_t = (\abar I)^\frac{1}{2}$
   \STATE $\hat{x}_{t+1} = \text{Project}_{\hFs, A_t}(\hat{u}_{t+1})$
   \STATE $x_{t+1} = A^{-1}\hat{x}$
   \ENDFOR
\end{algorithmic}
\end{algorithm}
 \end{minipage}
 \hfill
\end{figure*}

\subsection {Adaptive diagonal regularization}
\label{sec:per_coord_rate}

In this section, we introduce and analyze \FTPRLdiag, a specialization
of \FTPRL that uses regularization matrices from $\diagQ$.  Let $D_i =
\max_{x, x' \in \Fs} \abs{x_i - x'_i}$, the width of $\Fs$ along the
$i$th coordinate.  We construct a bound on the regret of
\FTPRLdiag in terms of these $D_i$.  The $D_i$ implicitly define a
hyperrectangle that contains $\Fs$.  When $\Fs$ is in fact such a
hyperrectangle, our bound is $\sqrt{2}$-competitive with the best
post-hoc optimal bound using matrices from $\diagQ$.

\begin {corollary} \label{cor:percoord}
Let $\Fs$ be a convex feasible set of width $D_i$ in coordinate $i$.
We can construct diagonal matrices $Q_t$ such that the $i$th entry on
the diagonal of $Q_{1:t}$ is given by:
\[
\lbar_{t,i} = \frac 2 D_i {\sqrt {\sum_{\tau=1}^{t} g^2_{\tau, i}}}.
\]
Then the regret of \FTPRL satisfies
\[
  \Regret \leq 2 \sum_{i=1}^n D_i \sqrt{\sum_{t=1}^T g^2_{t, i}}.
\]
When $\Fs$ is a hyperrectangle, then this algorithm is
$\sqrt{2}$-competitive with the post-hoc optimal choice of $Q_t$ from
the $\diagQ$. That is,
\[ \Regret \leq \sqrt{2} 
\inf_{Q \in \diagQ} \left( \max_{\hy \in \Fsym} \oset { \frac{1}{2} \hy^\tp Q\hy }
             + \sum_{t=1}^T g_t^\tp Q^{-1}g_t\right).\]

\end {corollary}
\begin{proof}  
The construction of $Q_{1:t}$ in the theorem statement implies
$(Q_t)_{ii} = \lambda_{t, i} \equiv \lbar_{t, i} - \lbar_{t-1, i}$.  These
entries are guaranteed to be non-negative as $\lbar_{t,i}$ is a
non-decreasing function of $t$.

We begin from Equation~\eqref{eq:brfunc}, letting $\hy_t$ be an
arbitrary sequence of points from $\Fsym$. For the left-hand term,
\[ 
   \frac{1}{2} \sum_{t=1}^T \hy_t^\tp Q_t\hy_t 
  = \frac{1}{2}\sum_{t=1}^T \sum_{i=1}^n \hy_{t,i}^2 \lambda_{t,i}
  \leq \frac{1}{2}\sum_{i=1}^n D_i^2 \sum_{t=1}^T\lambda_{t,i}
  = \frac{1}{2}\sum_{i=1}^n D_i^2 \lbar_{T,i}
  = \sum_{i=1}^n D_i \sqrt{\sum_{t=1}^T g^2_{t, i}}.
\] 
For the right-hand term, we have 
\[
 \sum_{t=1}^T g_t^\tp Q_{1:t}^{-1}g_t
  = \sum_{t=1}^T \sum_{i=1}^n \frac{g_{t,i}^2}{\lbar_{t,i}}
  = \sum_{i=1}^n \frac{D_i}{2}\sum_{t=1}^T  
       \frac{g_{t,i}^2}{\sqrt {\sum_{\tau=1}^{t} g^2_{\tau, i}}}
 \leq \sum_{i=1}^n D_i \sqrt{\sum_{t=1}^T g_{t,i}^2},
\]
where the last inequality follows from Lemma~\ref{lem:sum}.  Summing these
bounds on the two terms of Equation~\eqref{eq:brfunc}
yields the stated bound on regret.

Now, we consider the case where the feasible set is exactly a
hyperrectangle, that is, $\Fs = \set{x
\mid x_i \in [a_i, b_i]}$ where $D_i = b_i - a_i$.  Then,
the optimization of Equation~\eqref{eq:fixedQ} decomposes on a
per-coordinate basis, and in particular there exists
a $\hy \in \Fsym$ so that $\hy_i^2 = D_i^2$ in each coordinate.
Thus, for $Q = \text{diag}(\lambda_1, \dots, \lambda_n)$,
the bound function is exactly
\[\sum_{i=1}^n \frac{1}{2} \lambda_i D_i^2 
+ \frac{1}{\lambda_i}\sum_{t=1}^T g_{t,i}^2.
\]
Choosing $\lambda_i =
\frac{1}{D_i}\sqrt{2\sum_{t=1}^T g^2_{t, i}}$ minimizes this quantity, producing a
post-hoc bound of
\[\sqrt{2}\sum_{i=1}^n D_i \sqrt{\sum_{t=1}^T g_{t,i}^2},\]
verifying that the adaptive scheme is $\sqrt{2}$-competitive with
matrices from $\diagQ$.

\end{proof}

The regret guarantees of the \FTPRLdiag algorithm hold on arbitrary
feasible sets, but the competitive guarantee only applies for
hyperrectangles.  We now extend this result, showing that a
competitive guarantee can be made based on how well the feasible set
is approximated by hyperrectangles.

\newcommand{\Hout}{H^{\text{out}}}
\newcommand{\Hin}{H^{\text{in}}}

\begin{theorem}\label{thm:arbcomp}
Let $\Fs$ be an arbitrary feasible set, bounded by a hyperrectangle
$\Hout$ of width $W_i$ in coordinate $i$; further, let $\Hin$ be a
hyperrectangle contained by $\Fs$, of width $w_i > 0$ in coordinate $i$.
That is, $\Hin \subseteq \Fs \subseteq \Hout$.  Let $\beta = \max_i
\frac{W_i}{w_i}$.  Then, the \FTPRLdiag
is $\sqrt{2}\beta$-competitive with $\diagQ$ on $\Fs$.
\end{theorem}

\begin{proof}
By Corollary~\ref{cor:percoord}, the adaptive algorithm achieves
regret bounded by $2 \sum_{i=1}^n W_i \sqrt{\sum_{t=1}^T g^2_{t, i}}$.
We now consider the best post-hoc bound achievable with diagonal
matrices on $\Fs$.  Considering Equation~\eqref{eq:fixedQ}, it is
clear that  for any $Q$,
\[
 \quad \max_{y \in \Fsym} \frac{1}{2} y^\tp Qy
             + \sum_{t=1}^T g_t^\tp Q^{-1}g_t
\geq  
 \quad \max_{y \in \Hin_{\text{sym}}} \frac{1}{2} y^\tp Qy
             + \sum_{t=1}^T g_t^\tp Q^{-1}g_t,
\]
since the feasible set for the maximization ($\Fsym$) is larger on the
left-hand side.  But, on the right-hand side we have the post-hoc
bound for diagonal regularization on a hyperrectangle, which we
computed in the previous section to be
$\sqrt{2} \sum_{i=1}^n w_i \sqrt{\sum_{t=1}^T g^2_{t, i}}$.
Because $w_i \ge \frac {W_i} {\beta}$ by assumption, this is lower bounded by 
$\frac {\sqrt{2}} {\beta} \sum_{i=1}^n W_i \sqrt{\sum_{t=1}^T g^2_{t, i}}$, 
which proves the theorem.
\end{proof}

Having had success with $L_\infty$, we now consider the potential
benefits of diagonal adaptation for other $L_p$-balls.

\subsection{A post-hoc bound for diagonal regularization on $L_p$ balls}
\label{sec:lp_posthoc}

Suppose the feasible set $F$ is an unit $L_p$-ball, that is $F =
\set{x \mid \norm{x}_p \le 1}$.  We consider the post-hoc bound
optimization problem of Equation~\eqref{eq:fixedQF} with $\Qfs =
\diagQ$.  Our results are summarized in the following theorem.

\begin{theorem}\label{thm:lp_posthoc}
For $p > 2$, the optimal regularization matrix for $B_R$ in $\diagQ$
is not coordinate-constant (i.e., not contained in $\constQ$), except
in the degenerate case where $G_i \equiv \sum_{t=1}^T g_{t, i}^2$ is
the same for all $i$.  However, for $p \le 2$, the optimal
regularization matrix in $\diagQ$ always belongs to $\constQ$.
\end{theorem}
\begin {proof}
Since $\Fs$ is symmetric, the optimal post-hoc choice will be in the form of
Equation~\eqref{eq:fixedQF}.  Letting $Q = \text{diag}(\lambda_1,
\dots, \lambda_n)$, we can re-write this optimization problem as 
\begin {equation} \label {eq:lp_bound_opt}
\max_{ y : \norm{y}_p \leq 1} \oset { 2 \sum_{i=1}^n y_i^2 \lambda_i } 
   + \sum_{i=1}^n \frac {G_i} {\lambda_i} \mbox { .}
\end {equation}
To determine the optimal
$\lambda$ vector, we first derive a closed form for the solution to
the maximization problem on the left hand side, assuming $p \ge 2$ (we handle the case
$p < 2$ separately below).
First note that the
inequality $\norm{y}_p \le 1$ is equivalent to $\sum_{i=1}^n
\abs{y_i}^p \le 1$.  Making the change of variable $z_i =
y_i^2$, this is equivalent to $\sum_{i=1}^n z_i^{\frac p 2} \le
1$, which is equivalent to $\norm{z}_{\frac p 2} \le 1$
(the assumption $p \ge 2$ ensures that $\norm{\cdot}_{\frac p 2}$ is a norm).
Hence, the left-hand side optimization reduces to
\[
\max_{ z : \norm{z}_{\frac{p}{2}} \leq 1} 2 \sum_{i=1}^n z_i \lambda_i
 = 2 \norm{\lambda}_{q},
\]
where $q = \frac {p} {p-2}$, so that $\norm{\cdot}_{\frac p 2}$ and $\norm{\cdot}_q$ are
dual norms (allowing $q = \infty$ for $p = 2$).
Thus, for $p \ge 2$, the above bound simplifies to
\begin {equation} \label {eq:big_p_post}
B(\lambda) = 2\norm{\lambda}_q + \sum_{i=1}^n \frac {G_i} {\lambda_i}.
\end {equation}

First suppose $p > 2$, so that $q$ is finite.  Then, taking the
gradient of $B(\lambda)$,
\[
\nabla B(\lambda)_i = \frac 2 q \paren { 
  \sum_{i=1}^n \lambda_i^q }^{\frac 1 q - 1} \cdot q \lambda_i^{q-1} 
     - \frac {G_i} {\lambda_i^2}
  = 2\paren { \frac {\lambda_i} {\norm{\lambda}_q} }^{q-1} 
     - \frac {G_i} {\lambda_i^2},
\]
using $\frac{1}{q} - 1 = -\frac{1}{q}(q -1)$.  If we make
all the $\lambda_i$'s equal (say, to $\lambda_1$), then for the left-hand side we get
\[
\paren { \frac {\lambda_i} {\norm{\lambda}_q} }^{q-1} 
  = \paren { \frac {\lambda_1} {(n \lambda_1^q)^{\frac 1 q}} }^{q-1} 
  = \paren {\frac {1} {n^{\frac 1 q}}}^{q - 1} = n^{\frac 1 q - 1} \mbox { .}
\]
Thus the $i^{th}$ component of the gradient is $2n^{\frac 1 q - 1} -
\frac {G_i} {\lambda_1^2}$, and so if not all the $G_i$'s are
equal, some component of the gradient is non-zero.  Because
$B(\lambda)$ is differentiable and the $\lambda_i \geq 0$ constraints
cannot be tight (recall $g_1 > 0$), this implies a constant
$\lambda_i$ cannot be optimal, hence the optimal regularization matrix
is not in $\constQ$.

For $p \in [1, 2]$, we show that the solution to Equation~\eqref{eq:lp_bound_opt} is
\begin{equation}\label{eq:l_one_two_post}
B_\infty(\lambda) \equiv 2\norm{\lambda}_\infty + \sum_{i=1}^n \frac {G_i} {\lambda_i}.
\end{equation}
For $p = 2$ this follows immediately from
Equation~\eqref{eq:big_p_post}, because when $p = 2$ we have $q =
\infty$.  For $p \in [1, 2)$, the solution to
Equation~\eqref{eq:lp_bound_opt} is at least $B_\infty(\lambda)$,
because we can always set $y_i = 1$ for whatever $\lambda_i$ is
largest and set $y_j = 0$ for $j \neq i$.  If $p < 2$ then the
feasible set $\Fs$ is a subset of the unit $L_2$ ball, so the solution
to Equation~\eqref {eq:lp_bound_opt} is upper bounded by the solution
when $p = 2$, namely $B_\infty(\lambda)$.  It follows that the
solution is exactly $B_\infty(\lambda)$.  Because the left-hand term
of $B_\infty(\lambda)$ only penalizes for the largest $\lambda_i$, and
on the right-hand we would like all $\lambda_i$ as large as possible,
a solution of the form $\lambda_1 =
\lambda_2 = \dots = \lambda_n$ must be optimal.
\end {proof}

\subsection {Full matrix regularization on hyperspheres and hyperellipsoids}

In this section, we develop an algorithm for feasible sets $\Fs
\subseteq \set{x \mid \norm{Ax}_p \leq 1}$, where $p \in [1,2]$ and $A
\in \Snpp$.  When $\Fs = \set{x \mid \norm{Ax}_2 \leq 1}$, this
algorithm, \FTPRLscale, is $\sqrt{2}$-competitive with arbitrary $Q
\in \Snp$.  For $\Fs =
\set{x
\mid \norm{Ax}_p \leq 1}$ with $p \in [1,2)$ it is
$\sqrt{2}$-competitive with $\diagQ$.

First, we show that rather than designing adaptive schemes
specifically for linear transformations of norm balls, it is
sufficient (from the point of view of analyzing
\FTPRL) to consider unit norm balls if suitable pre-processing is
applied.  In the same fashion that pre-conditioning may speed batch
subgradient descent algorithms, we show this approach can produce
significantly improved regret bounds when $A$ is poorly conditioned
(i.e., the ratio of the largest to smallest eigenvalue is large).

\newcommand{\Sh}{A}
\newcommand{\Snh}{A^{-1}}

\begin{theorem}\label{thm:blackbox}
Fix an arbitrary norm $\norm{\cdot}$, and define an online linear
optimization problem $\inst = (\Fs, (g_1,
\dots, g_T))$ where $\Fs = \set{x \mid \norm{Ax} \leq 1}$ with $A
\in \Snpp$.  We define the  related instance
$\hinst = (\hFs, (\hg_1, \dots, \hg_T))$, where $\hFs = \set{\hx \mid
\norm{ \hx }
\leq 1}$ and $\hg_t = \Snh g_t$.  Then:
\begin{itemize}
 \item If we run any algorithm dependent only on subgradients on
  $\hinst$, and it plays $\hx_1, \dots, \hx_T$, then by playing the
  corresponding points $x_t = A^{-1}\hx_t$ on $\inst$ we achieve
  identical loss and regret.

  \item The post-hoc optimal bound over arbitrary $Q \in \Snpp$ is
  identical for these two instances. 
\end{itemize}
\end{theorem}

\begin{proof}
First, we note that for any function $h$ where $\min_{x : \norm{Ax} \leq 1} h(x)$ exists,
\begin{equation}\label{eq:equivmin}
   \min_{x : \norm{Ax} \leq 1} h(x) 
     = \min_{\hx : \norm{\hx} \leq 1} h(A^{-1}\hx),
\end{equation}
using the change of variable $\hx = Ax$.
For the first claim, note that $\hg_t^\tp = g_t^\tp A^{-1}$, and so
for all $t$, $\hg_t^\tp \hx_t = g_t^\tp A^{-1} A x_t = g_t^\tp x_t$,
implying the losses suffered on $\hinst$ and $\inst$ are identical.
Applying Equation~\eqref{eq:equivmin}, we have
\[
\min_{x: \norm{A x} \le 1} g_{1:t}^\tp x
  = \min_{\hx: \norm{\hx} \le 1} g_{1:t}^\tp A^{-1}\hx
  = \min_{\hx: \norm{\hx} \le 1} \hg_{1:t}^\tp \hx,
\]
implying the post-hoc optimal feasible points for the two instances
also incur identical loss.  Combining these two facts proves the
first claim.
For the second claim, it is sufficient to show for any $Q \in \Snpp$
applied to the post-hoc bound for problem $\inst$, there exists a
$\hat{Q} \in \Snpp$ that achieves the same bound for $\hinst$ (and
vice versa).  Consider such a $Q$ for $\inst$.  Then, again applying
Equation~\eqref{eq:equivmin}, we have
\begin{align*}
  \max_{y : \norm{Ay}_p \leq 1} \oset{2y^\tp Qy} 
      + \sum_{t=1}^T g_t^\tp Q^{-1}g_t
= \max_{\hy:\norm{\hy} \leq 1} \oset{2\hy^\tp A^{-1}Q A^{-1} \hy}
      + \sum_{t=1}^T \hg_t^\tp AQ^{-1}A\hg_t.
\end{align*}
The left-hand side is the value of the post-hoc bound on $\inst$ from
Equation~\eqref{eq:fixedQF}.  Noting that $(A^{-1}Q A^{-1})^{-1} =
AQ^{-1}A$, the right-hand side is the value of the post hoc bound for
$\hinst$ using $\hat{Q} = A^{-1}Q A^{-1}.$ The fact $A^{-1}$ and $Q$
are in $\Snpp$ guarantees $\hat{Q} \in \Snpp$ as well, and the theorem
follows.
\end{proof}

We can now define the adaptive algorithm \FTPRLscale: given a $\Fs
\subseteq \set{x \mid \norm{A x}_p \leq 1}$, it uses the
transformation suggested by Theorem~\ref{thm:blackbox}, applying the
coordinate-constant algorithm of Corollary~\ref{cor:global} to the
transformed instance, and playing the corresponding point mapped back
into $\Fs$.\footnote{ By a slightly more cumbersome argument, it is
possible to show that instead of applying this transformation, \FTPRL
can be run directly on $\Fs$ using appropriately transformed $Q_t$
matrices.}  Pseudocode is given as Algorithm~\ref{alg:scale}.

\begin{theorem}\label{thm:sphere}
The diagonal-constant algorithm analyzed in Corollary~\ref{cor:global}
is $\sqrt{2}$-competitive with $\Snp$ when $\Fs = \set{x \mid
\norm{x}_p \leq 1}$ for $p=2$, and $\sqrt{2}$-competitive against
$\diagQ$ when $p \in [1,2)$.  Furthermore, when $\Fs = \set{x \mid
\norm{Ax}_p \leq 1}$ with $A \in
\Snpp$, the \FTPRLscale algorithm (Algorithm~\ref{alg:scale}) achieves
these same competitive guarantees.  In particular, when $\Fs = \set{x
\mid \norm{x}_2 \leq 1}$, we have
\[ \Regret \leq \sqrt{2} 
\inf_{Q \in \Snp} \left( \max_{y \in \Fs} \oset { 2 y^\tp Qy }
             + \sum_{t=1}^T g_t^\tp Q^{-1}g_t\right).\]
\end{theorem}

\begin{proof}
The results for $\diagQ$ with $p \in [1, 2)$ follow from
Theorems~\ref{thm:lp_posthoc} and \ref{thm:blackbox} and
Corollary~\ref{cor:global}.  We now consider the case $p=2$.  Consider
a $Q \in \Snpp$ for Equation~\eqref{eq:fixedQF} (recall only a
$Q \in \Snpp$ could be optimal since $g_1 > 0$).  We can write $Q = PDP^\tp $
where $D= \text{diag}(\lambda_1, \dots,
\lambda_n)$ is a diagonal matrix of positive eigenvalues and $PP^\tp =
I$.  It is then easy to verify $Q^{-1} = PD^{-1}P^\tp $.

When $p=2$ and $\Fs = \set{x \mid\norm{x}_p \leq 1}$,
Equation~\eqref{eq:l_one_two_post} is tight, and so the post-hoc bound
for $Q$ is
\[2\max_i (\lambda_i) + \sum_{t=1}^T g_t^\tp (P D^{-1} P^\tp ) g_t.\]
Let $z_t = P^\tp g_t$, so each right-hand term is 
$\sum_{i=1}^n \frac{z_{t,i}^2}{\lambda_i}$.  It is clear this quantity is
minimized when each $\lambda_i$ is chosen as large as possible, while on the
left-hand side we are only penalized for the largest eigenvalue of $Q$
(the largest $\lambda_i$).  Thus, a solution where $D
=
\alpha I$ for $\alpha >0$ is optimal.  Plugging  into
the bound, we have
\begin{align*}
B(\alpha) 
   &= 2\alpha + \sum_{t=1}^T g_t^\tp \paren{P \paren{\frac{1}{\alpha}I}P^\tp } g_t  
    = 2\alpha + \frac{1}{\alpha}\sum_{t=1}^T g_t^\tp g_t  
    = 2\alpha + \frac{G_T}{\alpha}
\end{align*}
where we have used the fact that $PP^\tp = I$.  Setting $\alpha =
\sqrt{G_T/2}$ produces a minimal post-hoc bound of $2\sqrt{2 G_T}$.
The diameter $D$ is 2, so the coordinate-constant algorithm
has regret bound $4 \sqrt{G_T}$ (Corollary~\ref{cor:global}), proving
the first claim of the theorem for $p=2$.  The second claim
follows from Theorem~\ref{thm:blackbox}.
\end{proof}

Suppose we have a problem instance where $\Fs = \set{x \mid \norm{A
x}_2 \leq 1}$ where $A = \text{diag}(1/a_1, \dots, 1/a_n)$ with $a_i >
0$.  To demonstrate the advantage offered by this transformation, we
can compare the regret bound obtained by directly applying the
algorithm of Corollary~\ref{cor:global} to that of the $\FTPRLscale$
algorithm.  Assume WLOG that $\max_i a_i = 1$, implying the diameter
of $\Fs$ is 2.  Let $g_1, \dots, g_T$ be the loss functions for this
instance.  Recalling $G_i = \sum_{t=1}^T g_{t,i}^2$, applying
Corollary~\ref{cor:global} directly to this problem gives 
\begin{equation}\label{eq:badbound}
\Regret \leq 4\sqrt{ \sum_{i=1}^n  G_i}.
\end{equation}
This is the same as the bound obtained by online subgradient descent
and related algorithms as well.

We now consider
\FTPRLscale, which uses the transformation of
Theorem~\ref{thm:blackbox}.  Noting $D=2$ for the hypersphere and
applying Corollary~\ref{cor:global} to the transformed problem gives
an adaptive scheme with
\[ 
  \Regret
  \leq 4\sqrt{ \sum_{i=1}^n \sum_{t=1}^T\hg_{t,i}^2}
     = 4\sqrt{ \sum_{i=1}^n  a_i^2 \sum_{t=1}^Tg_{t,i}^2}
     = 4\sqrt{ \sum_{i=1}^n  a_i^2 G_i}.
\]
This bound is never worse than the bound of
Equation~\eqref{eq:badbound}, and can be arbitrarily better when many
of the $a_i$ are much smaller than 1.

\section{Related work}\label{sec:related}

In the batch convex optimization setting, it is well known that
convergence rates can often be dramatically improved through the use
of preconditioning, accomplished by an appropriate change of
coordinates taking into account both the shape of the objective
function and the feasible region~\cite{boyd}.  To our knowledge, this
is the first work that extends these concepts (necessarily in a quite
different form) to the problem of online convex optimization, where
they can provide a powerful tool for improving regret (the online
analogue of convergence rates).

Perhaps the closest algorithms in spirit to our diagonal adaptation
algorithm are confidence-weighted linear
classification~\cite{drezde08} and AROW \cite{crammer09}, in that they
make different-sized adjustments for different coordinates.  Unlike
our algorithm, these algorithms apply only to classification problems
and not to general online convex optimization, and the guarantees are
in the form of mistake bounds rather than regret bounds.

\FTPRL is similar to the lazily-projected gradient descent
algorithm of~\cite[Sec. 5.2.3]{zinkevich04thesis}, but with
a critical difference: the latter
effectively centers regularization outside of the current
feasible region (at $u_t$ rather than $x_t$).
As a consequence, lazily-projected gradient descent
only attains low regret via a re-starting mechanism or a constant
learning rate (chosen with knowledge of $T$).  It is our technique of
always centering additional regularization inside the feasible set
that allows us to make guarantees for adaptively-chosen
regularization.

Most recent state-of-the-art algorithms for online learning are in
fact general algorithms for online convex optimization applied to
learning problems.  Many of these algorithms can be thought of as
(significant) extensions of online subgradient descent, including
~\cite{duchi09fobos,do09proximal,shwartz07pegasos}.  Apart from the very
general work of~\cite{kv03}, few general follow-the-regularized-leader
algorithms have been analyzed, with the notable exception of the
recent work of~\emcite{xiao09dualaveraging}.

The notion of proving competitive ratios for regret bounds that are
functions of regularization parameters is not unique to this paper.
\emcite{bartlett08} and \emcite{do09proximal} proved guarantees of
this form, but for a different algorithm and class of regularization
parameters.

In concurrent work \cite{streeter10percoord}, the
authors proved bounds similar to those of Corollary~\ref{cor:percoord}
for online gradient descent with per-coordinate learning rates.  These
results were significantly less general that the ones presented here,
and in particular were
restricted to the case where $\Fs$ was exactly a hyperrectangle.  The
\FTPRL algorithm and bounds proved in this paper hold for arbitrary
feasible sets, with the bound depending on the shape of the feasible set
as well as the width along each dimension.
Some results similar to those in this work
were developed concurrently by~\emcite{duchi10}, though for a
different algorithm and using different analysis techniques.

\section{Conclusions}
In this work, we analyzed a new algorithm for online convex
optimization, which takes ideas both from online subgradient descent
as well as follow-the-regularized-leader.  In our
analysis of this algorithm, we show that the learning rates that occur
in standard bounds can be replaced by positive semidefinite matrices.
The extra degrees of freedom offered by these generalized learning
rates provide the key to proving better regret bounds.  We
characterized the types of feasible sets where this technique can lead
to significant gains, and showed that while it
does not help on the hypersphere, it can have dramatic impact when the
feasible set is a hyperrectangle.

The diagonal adaptation algorithm we introduced can be viewed as an
incremental optimization of the formula for the final bound on regret.
In the case where the feasible set really is a hyperrectangle, this
allows us to guarantee our final regret bound is within a small
constant factor of the best bound that could have been obtained had
the full problem been known in advance.  The diagonal adaptation
algorithm is efficient, and exploits exactly the kind of structure
that is typical in large-scale real-world learning problems such as
click-through rate prediction and text classification.

Our work leaves open a number of interesting directions for future
work, in particular the development of competitive algorithms for
arbitrary feasible sets (without resorting to bounding norm-balls),
and the development of algorithms that optimize over richer families
of regularization functions.

\begin{small}
\setlength{\itemsep}{0.5mm}
\bibliographystyle{\mybibstyle}
\bibliography{colt10}
\end{small}

\shortlong{}{
\appendix

\section{A Proof of the FTRL Bound}\label{sec:ftrl_appendix}

In this section we provide a proof of Lemma \ref {lem:ftrl}.  The
high-level structure of our proof follows Kalai and Vempala's analysis
of the \emph {follow the perturbed leader} algorithm, in that we prove
bounds on three quantities:
\begin {enumerate}
\item the regret of a hypothetical \emph {be the leader} algorithm
  (BTL), which on round $t$ plays
\[
x^*_t = \argmin_{x \in \Fs} f_{1:t}(x) ,
\]

\item the difference between the regret of BTL and that of the \emph
  {be the regularized leader} algorithm (BTRL), which plays
\begin{equation}\label{eq:btrlx}
\hx_t = \argmin_{x \in \Fs} \oset {r_{1:t}(x) + f_{1:t}(x)  } = \xti,
\end{equation}
and
\item the difference between the regret of BTRL and that of FTRL.
\end {enumerate}

As shown in \cite {kv03}, the BTL algorithm has regret $\le 0$ even without any restrictions on the loss functions or the feasible set.  The proof is a straightforward induction, which we reproduce here for completeness.
\begin {lemma} [\cite{kv03}] \label {lem:btl}
Let $f_1, f_2, \ldots, f_T$ be an arbitrary sequence of functions,
and let $\Fs$ be an arbitrary set.  Define $x^*_t \equiv \argmin_{x \in
  \Fs} \sum_{\tau=1}^t f_\tau(x)$.  Then
\[
\sum_{t=1}^T f_t(x^*_t) \le \sum_{t=1}^T f_t(x^*_T) \mbox { .}
\]
\end {lemma}

\begin {proof}
We prove this by induction on $T$.  For $T=1$ it is trivially true.
Suppose that it holds for $T-1$.  Then
\begin {align*}
\sum_{t=1}^T f_t(x^*_t) & = f_T(x^*_T) + \sum_{t=1}^{T-1} f_t(x^*_t) \\
& \le f_T(x^*_T) + \sum_{t=1}^{T-1} f_t(x^*_{T-1})
& \mbox{(Induction hypothesis)} \\
& \le f_T(x^*_T) + \sum_{t=1}^{T-1} f_t(x^*_{T})
& \mbox {(Definition of $x^*_{T-1}$)} \\
& = \sum_{t=1}^{T} f_t(x^*_{T}) \mbox { .}
\end {align*}
\end {proof}

We next prove a bound on the regret of BTRL.

\begin {lemma} \label {lem:btrl}
Let $r_1, r_2, \ldots, r_T$ be a sequence of non-negative functions.
Then BTRL, which on round $t$
plays $\hx_t$ as defined by equation \eqref {eq:btrlx}, has regret at
most $r_{1:T}(\xs)$ where $\xs$ is the post-hoc optimal solution.
\end {lemma}

\begin {proof}
Define $f_t'(x) = f_t (x) + r_t(x)$.  Observe that $\hx_t = \argmin_{x \in \Fs} f'_{1:t}(x)$.
  Thus, by Lemma \ref {lem:btl}, we have
\[
\sum_{t=1}^T f_t'(\hx_t) \le \min_{x \in \Fs} f'_{1:T}(x) \le f'_{1:T}(\xs)
\]
or equivalently,
\[
\sum_{t=1}^T f_t (\hx_t) + r_t(\hx_t)
  \leq r_{1:T}(\xs) + f_{1:T}(\xs) .
\]
Dropping the non-negative $r_t(\hx_t)$ terms on the left
hand side proves the lemma.
\end {proof}

By definition, the total loss of FTRL (which plays $\xt$) exceeds that
of BTRL (which plays $\hx_t = \xti$) by $\sum_{t=1}^T f_t (\xt) - f_t(\xti)$.
Putting these facts together proves Lemma \ref {lem:ftrl}.

\section{Proof of Lemma~\ref{lem:sum}}\label{ap:lem_sum_proof}
\begin {proof}
The lemma is clearly true for $n = 1$.  Fix some $n$, and assume the
lemma holds for $n - 1$.  Thus,
\begin {align*}
	\sum_{i=1}^n \frac { x_i } { \sqrt { \sum_{j=1}^i x_j } }
& \le  2 \sqrt { \sum_{i=1}^{n-1} x_i }
  + \frac { x_n } { \sqrt { \sum_{i=1}^n x_i } } \\
& = 2 \sqrt {Z - x} + \frac {x} { \sqrt { Z } }
\end {align*}
where we define $Z = \sum_{i=1}^{n} x_i$ and $x = x_n$.  The
derivative of the right hand side with respect to $x$ is $\frac {-1}
{\sqrt {Z -x}} + \frac {1} {\sqrt Z}$, which is negative for $x > 0$.
Thus, subject to the constraint $x \ge 0$,
the right hand side is maximized at $x = 0$, and is
therefore at most $2 \sqrt Z$.
\end {proof}
} %

\end{document}